\documentclass[english]{gretsi}

\usepackage[utf8]{inputenc}

\usepackage[T1]{fontenc}

\usepackage{amsmath,amssymb,amsfonts}
\usepackage{bbm}
\usepackage{amsthm}
\usepackage{cleveref}

\usepackage{graphicx}
\usepackage{float}
\usepackage[font=it]{caption}
\graphicspath{ {./figs/} }

\usepackage{xcolor}
\definecolor{linkcolor}{RGB}{83,83,182}
\definecolor{citecolor}{RGB}{128,0,128}
\usepackage{hyperref}
\hypersetup{
    colorlinks=true,
    citecolor=citecolor,
    linkcolor=linkcolor,
    urlcolor=linkcolor
}

\DeclareMathOperator{\logdet}{logdet}
\DeclareMathOperator*{\argmin}{argmin~}

\DeclareMathOperator{\prox}{prox}
\DeclareMathOperator{\Id}{Id}

\DeclareMathOperator{\vectorize}{vec}
\DeclareMathOperator{\sign}{sign}

\newcommand{\norm}[1]{\Vert #1 \Vert}
\newcommand{\abs}[1]{\vert #1 \vert}

\newcommand{\bbR}{\mathbb{R}}

\newcommand{\cC}{\mathcal{C}}

\newcommand{\cL}{\mathcal{L}}

\newcommand{\cS}{\mathcal{S}}
\newcommand{\Thetab}{\boldsymbol{\Theta}}
\newcommand{\Db}{\mathbf{D}}
\newcommand{\Eb}{\mathbf{E}}
\newcommand{\Jb}{\mathbf{J}}
\newcommand{\Kb}{\mathbf{K}}
\newcommand{\Mb}{\mathbf{M}}
\newcommand{\Sb}{\mathbf{S}}
\newcommand{\Xb}{\mathbf{X}}
\newcommand{\Zb}{\mathbf{Z}}
\newcommand{\Lb}{\mathbf{\Lambda}}
\newcommand{\Thetahat}{\hat{\Thetab} (\lambda)}
\newcommand{\dr}{\mathrm{d}}

\newcommand\blfootnote[1]{%
  \begingroup
  \renewcommand\thefootnote{}\footnote{#1}%
  \addtocounter{footnote}{-1}%
  \endgroup
}

\usepackage{xcolor}

\usepackage[numbers]{natbib}

\newtheorem{theorem}{Theorem}
\newtheorem{proposition}[theorem]{Proposition}
\newtheorem{assumption}[theorem]{Assumption}

\begin{document}


\titre{Implicit differentiation for hyperparameter tuning\\the weighted Graphical Lasso}

\auteurs{
  \auteur{Can}{Pouliquen}{can.pouliquen@ens-lyon.fr}{}
  \auteur{Paulo}{Gon\c{c}alves}{paulo.goncalves@inria.fr}{}
  \auteur{Mathurin}{Massias}{mathurin.massias@inria.fr}{}
  \auteur{Titouan}{Vayer}{titouan.vayer@inria.fr}{}
}

\affils{
  \affil{}{Univ Lyon, Ens Lyon, UCBL, CNRS, Inria, LIP, F-69342, LYON Cedex 07, France.
  }
}


\resume{Nous dérivons les résultats mathématiques nécessaires à l'implémentation d'une procédure de calibration d'hyperparamètres pour le Graphical Lasso via un problème d'optimisation bi-niveau résolu par méthode du premier-ordre.
  En particulier, nous dérivons la Jacobienne de la solution du Graphical Lasso par rapport à ses hyperparamètres de régularisation.}

\abstract{We provide a framework and algorithm for tuning the hyperparameters of the Graphical Lasso via a bilevel optimization problem solved with a first-order method. In particular, we derive the Jacobian of the Graphical Lasso solution with respect to its regularization hyperparameters.}

\maketitle


\section{Introduction} \label{sec:intro}
The Graphical Lasso estimator (GLASSO) \citep{banerjee2008model} is a commonly employed and established method for estimating sparse precision matrices.
It models conditional dependencies between variables by finding a precision matrix that maximizes the $\ell_1$-penalized log-likelihood of the data under a Gaussian assumption.
More precisely, the GLASSO is defined as\footnote{in variants, the diagonal entries of $\Thetab$ are not penalized. This is handled by the framework of \Cref{sec:matricial_regularization}}
\begin{equation}\label{eq:graphical_lasso}
  \hat\Thetab(\lambda) = \argmin_{\Thetab \succ 0}
  \underbrace{
    -\logdet(\Thetab) + \langle \Sb, \Thetab\rangle + \lambda \norm{\Thetab}_{1}
  }_{=\Phi(\Thetab, \lambda)}
  \, \,,
\end{equation}
where $\Sb = \frac{1}{n} \sum_{i=1}^n \mathbf{x}_i \mathbf{x}_i^\top \in \bbR^{p \times p}$ is the empirical covariance matrix of the data $(\mathbf{x}_1, \cdots, \mathbf{x}_n)$.
There exist many first- and second-order algorithms to solve this problem \citep{friedman2008sparse,hsieh2014quic,oztoprak2012newton}. These approaches all require choosing the right regularization hyperparameter $\lambda$ that controls the sparsity of $\hat\Thetab(\lambda)$.
This is a challenging task that typically involves identifying the value of $\lambda$ for which the estimate $\hat\Thetab(\lambda)$ minimizes a certain performance criterion $\cC$.
This problem can be framed as the bilevel optimization problem
\begin{equation}\label{eq:bilevel}
  \begin{aligned}
    \lambda^{\mathrm{opt}}
     & =  \argmin_\lambda \{ \cL(\lambda) \triangleq \cC(\hat\Thetab(\lambda) ) \}                       \\ 
     & \quad \quad  \text{s.t. } \hat\Thetab(\lambda) = \argmin_{\Thetab \succ 0} \Phi(\Thetab, \lambda)
    \,,
  \end{aligned}
\end{equation}
where the minimizations over $\lambda$ and $\Thetab$ are called respectively the \emph{outer} and \emph{inner} problems. The standard approach to tune the hyperparameter $\lambda$ is  \emph{grid-search}: for a predefined list of values for $\lambda$, the solutions of \eqref{eq:graphical_lasso} are computed and the one minimizing $\cC$ is chosen, which corresponds to solving \eqref{eq:bilevel} with a zero-order method.

In this paper, we propose instead a first-order method for \eqref{eq:bilevel}, relying on the computation of the so-called \emph{hypergradient} and the Jacobian of the GLASSO objective with respect to $\lambda$. Despite the non-smoothness of the inner problem, we derive a closed-form formula for the Jacobian.

Our main contributions are the derivations of the equations of implicit differentiation for the GLASSO: first in the single parameter regularization case for ease of exposure in \Cref{sec:problem_setting}, and then for matrix regularization in \Cref{sec:matricial_regularization}. Our work paves the way for a scalable approach to hyperparameter tuning for the GLASSO and its variants, and could naturally apply to more complex extensions of the GLASSO such as \citep{benfenati2020proximal}.
We provide open-source code for the reproducibility of our experiments which are treated in \Cref{sec:experiments}.
\vspace{-2mm}
\paragraph{Related work}
Although not strictly considering the GLASSO problem, some other alternatives to grid-search have been considered in the literature, including random search \Citep{bergstra2012random} or Bayesian optimization \Citep{snoek2012practical}. 
While we compute the hypergradient by implicit differentiation \Citep{bengio2000gradient}, automatic differentiation in forward and backward modes have also been proposed \Citep{franceschi2017forward}.
\vspace{-2mm}
\paragraph{Notation}
The set of integers from 1 to $k$ is $[k]$.
For a set $\cS \subset [p]$ and a matrix $\mathbf{A} \in \bbR^{p \times p}$, $\mathbf{A}_{:,S}$  (\textit{resp}. $\mathbf{A}_{S,:}$) is the restriction of $\mathbf{A}$ to columns (\textit{resp}. rows) in $\cS$.
The Kronecker and Hadamard products between two matrices are denoted by $\otimes$ and $\odot$ respectively.
The column-wise vectorization operation, transforming matrices into vectors, is denoted by $\vectorize(\cdot)$ and $\vectorize^{-1}(\cdot)$ denotes the inverse operation.
For a differentiable function $F$ of two variables, $\Jb_1 F$ and $\Jb_2 F$ denote the Jacobians of $F$ with respect to its first and second variable respectively.
A fourth-order tensor $\mathbf{A}$ applied to a matrix $\mathbf{B}$ corresponds to a contraction according to the last two indices: $(\mathbf{A} : \mathbf{B})_{ij} = \sum_{k,l} A_{ijkl}B_{kl}$.
The relative interior of a set $\cS$ is denoted by $\mathrm{relint}(\cS)$.

\section{The scalar case} \label{sec:problem_setting}

If the solution of the inner problem $\hat \Thetab(\lambda)$ is differentiable with respect to $\lambda$, the gradient of the outer objective function $\cL$, called \emph{hypergradient}, can be computed by the chain rule:
\begin{equation}
  \frac{\dr \mathcal{L}}{\dr \lambda}(\lambda) = \sum_{i,j=1}^p \frac{\partial \cC}{\partial \Theta_{ij}}(\Thetahat) \frac{\partial \hat{\Theta}_{ij}}{\partial \lambda}(\lambda) \, .
\end{equation}


\blfootnote{This work was partially funded by the AllegroAssai ANR-19-CHIA-0009 project.}

The hypergradient can then be used to solve the bilevel problem with a first-order approach such as gradient descent:
$\lambda_{k+1} = \lambda_k - \rho \frac{\dr \cL}{\dr \lambda}(\lambda_k) $.
The main challenge in the hypergradient evaluation is the computation of $\frac{\partial \hat \Theta_{ij}}{\partial \lambda}(\lambda)$, that we summarize in a $p\times p$ matrix\footnote{$\hat\Jb$ is the image by $\vectorize^{-1}$ of the Jacobian of $\lambda \mapsto \vectorize(\hat\Thetab(\lambda))$} $\hat\Jb= (\frac{\partial \hat \Theta_{ij}}{\partial \lambda}(\lambda))_{ij}$.
When the inner objective $\Phi$ is smooth, $\hat\Jb$ can be computed by differentiating the \emph{optimality condition} of the inner problem, $\nabla_{\Thetab} \Phi(\Thetab, \lambda) = 0$, with respect to $\lambda$, as in \Citep{bengio2000gradient}.

Unfortunately, in our case the inner problem is not smooth.
We however show in the following how to compute $\hat\Jb$ by differentiating a fixed point equation instead of differentiating the optimality condition as performed in \Citep{bertrand2022implicit} for the Lasso.
The main difficulty in our case stems from our optimization variable being a matrix instead of a vector, which induces the manipulation of tensors in the computation of $\hat\Jb$.
Let
\begin{equation}
  \begin{aligned}
    F : \bbR^{p \times p} \times \bbR_+ & \to \bbR^{p \times p}                               \\
    (\Zb, \lambda)                      & \mapsto \prox_{\gamma \lambda \norm{\cdot}_1} (\Zb)
  \end{aligned}
  \,,
\end{equation}
which is equal to the soft-thresholding operator
\begin{equation}\label{eq:st_def}
  F(\Zb, \lambda)=
  \sign(\Zb) \odot (\abs{\Zb} - \lambda \gamma)_+ \, ,
\end{equation}
where all functions apply entry-wise to $\Zb$. When $\Thetahat$ solves the inner problem \eqref{eq:graphical_lasso}, it fulfills a fixed-point equation related to proximal gradient descent.
Valid for any $\gamma>0$ \citep[Prop. 3.1.iii]{combettes2005signal}, this equation is as follows:
\begin{align} \label{eq:fixed_point}
   & \Thetahat = F(\Thetahat - \gamma (\Sb - \Thetahat^{-1}), \lambda) \, .
\end{align}
To compute $\hat\Jb$, the objective is now to differentiate \eqref{eq:fixed_point} with respect to $\lambda$. By defining $\hat{\Zb} \triangleq \Thetahat - \gamma(\Sb - \Thetahat^{-1})$ we will show that $F$ is differentiable at $(\hat \Zb, \lambda)$. Since $F$ performs entry-wise soft-thresholding, each of its coordinates is weakly differentiable \citep[Prop. 1, Eq. 32]{deledalle2014stein} and the only non-differentiable points are when $|\hat Z_{ij}| =  \lambda \gamma$. To ensure that none of the entries of $\hat \Zb$ take the value $\pm\lambda \gamma$, we will use the first-order optimality condition for the inner problem \eqref{eq:graphical_lasso}.{}

\begin{proposition}\label{prop:first_order}
  Let $\hat\Thetab(\lambda)$ be a solution of \eqref{eq:graphical_lasso}.
  Then, using Fermat's rule and the expression of the subdifferential of the $\ell_1$-norm \citep[Thm. 3.63, Ex. 3.41]{beck2017first},
  \begin{equation}\label{eq:first_order}
    [\Thetahat^{-1}]_{ij} - S_{ij} \in
    \begin{cases}
      \{ \lambda \sign \hat{\Theta}(\lambda)_{ij} \} , & \text{if } \,\hat{\Theta}(\lambda)_{ij} \neq 0 \, , \\
      [-\lambda, \lambda]                              & \text{ otherwise.}
    \end{cases}
  \end{equation}
\end{proposition}
We also require the following assumption which is classical (see \textit{e.g.} \citep[Thm 3.1]{liang2014local} and references therein).
\begin{assumption}[Non degeneracy] \label{ass:non-degeneracy}
  We assume that the inner problem is non-degenerated, meaning that it satisfies a slightly stronger condition than \eqref{eq:first_order}:
  \begin{equation}
    \Thetahat^{-1} - \Sb \in \mathrm{relint}\, \lambda \partial \norm{\cdot}_{1} \,.
  \end{equation}
  This implies that in \eqref{eq:first_order}, the interval $[-\lambda, \lambda]$ in the second case becomes $\left(-\lambda, \lambda\right)$.
\end{assumption}
Using \Cref{prop:first_order} under \Cref{ass:non-degeneracy}, we conclude that $|\hat{Z}_{ij}|$ never takes the value $\lambda \gamma$ so that \eqref{eq:fixed_point} is differentiable.
Indeed when $\hat{\Theta}(\lambda)_{ij} = 0, |[\Thetahat^{-1}]_{ij} - S_{ij}| <\lambda$ which implies $|\hat{Z}_{ij}| < \lambda \gamma$.
Conversely, when $\hat{\Theta}(\lambda)_{ij} \neq 0$, $\hat{Z}_{ij} = \hat{\Theta}(\lambda)_{ij} + \gamma \lambda \sign(\hat{\Theta}(\lambda)_{ij})$ which implies $|\hat{Z}_{ij}| > \lambda \gamma$.
Consequently, we can differentiate \Cref{eq:fixed_point} \textit{w.r.t.} $\lambda$, yielding
\begin{equation} \label{eq:jacobian_fixed_point}
  \hat\Jb =
  \Jb_1 F(\hat \Zb, \lambda)
  \left( \hat\Jb - \gamma \Thetahat^{-1} \hat\Jb \Thetahat^{-1} \right)
  + \Jb_2 F(\hat \Zb, \lambda) \,.
\end{equation}
The goal is now to solve \eqref{eq:jacobian_fixed_point} in $\hat \Jb$. We define $\Db \triangleq \Jb_1 F(\hat \Zb, \lambda)$ the Jacobian of $F$ with respect to its first variable at $(\hat \Zb, \lambda)$ which is represented by a fourth-order tensor in $\bbR^{p \times p \times p \times p}$:
\begin{align}
  D_{ijkl} = \left[\frac{\partial F}{\partial Z_{kl}}(\hat \Zb, \lambda)\right]_{ij}\,.
\end{align}
We also note $\Eb\triangleq \Jb_2 F(\hat \Zb, \lambda)$ viewed as a $p \times p$ matrix.

\vspace{-3mm}

\paragraph{Jacobian with respect to $Z$}
Because the soft-thresholding operator acts independently on entries, one has $D_{ijkl} = 0$ when $(i, j) \neq (k, l)$.
From \Cref{eq:st_def}, the remaining entries are given by
\begin{align}\label{eq:D_expression}
  D_{ijij} =
  \begin{cases}
    0 \, , & \text{if } \abs{\hat{Z}_{ij}} < \lambda \gamma \: \, , \\
    1 \, , & \text{otherwise.}
  \end{cases}
\end{align}

\vspace{-3mm}
\paragraph{Jacobian with respect to $\lambda$}
Similarly to $\Db$, $\Eb$ is given by
\begin{align}\label{eq:E_expression}
  E_{ij} = \left[\frac{\partial F}{\partial \lambda}(\hat\Zb, \lambda)\right]_{ij} =
  \begin{cases}
    0      \, ,                      & \text{if } \abs{\hat{Z}_{ij}} < \lambda \gamma \, , \\
    -\gamma \sign(\hat{Z}_{ij}) \, , & \text{otherwise.}
  \end{cases}
\end{align}
We can now find the expression of $\hat{\Jb}$ as described in the next proposition.


\begin{proposition}
  \label{prop:jacscalar}
  Let $\cS \subset[p^2]$ be the set of indices $i$ such that $\vectorize(|\hat \Zb|)_i > \lambda \gamma$.
  The Jacobian $\hat\Jb$ is given by
  \begin{equation*}
    \begin{aligned}
       & \vectorize\left(\hat\Jb\right)_\cS
      = \left[ \left( \Thetahat^{-1} \otimes \Thetahat^{-1} \right)_{\cS,\cS} \right]^{-1} \vectorize\left( \Eb \right)_\cS / \gamma \, , \\
       & \vectorize\left(\hat\Jb\right)_{\cS^c} = 0                                                                                       
      \, .
    \end{aligned}
  \end{equation*}
\end{proposition}

\begin{proof}
  From \eqref{eq:D_expression}, one has that $\Db$ applied to $\Xb \in \bbR^{p \times p}$ is simply a masking operator $\Db : \Xb = \Mb \odot \Xb$,
  where $M_{ij} = \mathbbm{1} \{|\hat Z_{ij}| < \lambda \gamma\}$.
  Thus \eqref{eq:jacobian_fixed_point} reads
  
  \begin{equation}\label{eq:full_eq_jac}
    \begin{aligned}
      \hat{\Jb} = \Mb \odot \left( \hat{\Jb} - \gamma \Thetahat^{-1} \hat{\Jb} \Thetahat^{-1} \right) +  \Eb \, .
    \end{aligned}
  \end{equation}
  Now by the expression of $\Eb$ \eqref{eq:E_expression}, $\Eb$ has the same support as $\Mb$, so $\Eb = \Mb \odot \Eb$, so $\hat\Jb = \Mb \odot \hat\Jb$, and \eqref{eq:full_eq_jac} simplifies to
  \begin{align}\label{eq:masked_eq_jac}
    \Mb \odot (\Thetahat^{-1} \hat{\Jb} \Thetahat^{-1}) = \Eb / \gamma \, .
  \end{align}
  Using the mixed Kronecker matrix-vector product property $\vectorize(\mathbf{A} \mathbf{C} \mathbf{B}^\top) = (\mathbf{A} \otimes \mathbf{B}) \vectorize(\mathbf{C})$, by vectorizing \eqref{eq:masked_eq_jac}, we get
  \begin{align}\label{eq:all_vectorized}
    \vectorize(\Mb) \odot (\Thetahat^{-1} \otimes \Thetahat^{-1}) \vectorize(\hat{\Jb})
    = \vectorize(\Eb) / \gamma \, .
  \end{align}
  Writing $\mathbf{K} = \Thetahat^{-1} \otimes \Thetahat^{-1}$, we have $\Kb \vectorize \hat{\Jb} =  \Kb_{:,\cS} (\vectorize \hat \Jb)_\cS$ because $\vectorize(\hat\Jb)$ is $0$ outside of $\cS$.
  Then, \eqref{eq:all_vectorized} can be restricted to entries in $\cS$, yielding
  $\Kb_{\cS,\cS} (\vectorize \hat\Jb)_{\cS} = (\vectorize \Eb)_\cS$, which concludes the proof.
\end{proof}

\vspace{-3mm}
\section{Matrix of hyperparameters}
\label{sec:matricial_regularization}
In the vein of \cite{van2019generalized}, we now consider the \emph{weighted} GLASSO where the penalty is controlled by a matrix of hyperparameters $\Lb \in \bbR^{p \times p}$. 
In the weighted GLASSO, $\lambda \norm{\Thetab}_{1}$ is replaced by
\begin{equation}
  \norm{\Lb \odot \Thetab}_1 = \sum_{k, l} \Lambda_{kl} \abs{\Theta_{kl}} \, ,
\end{equation}
with $\Lb = (\Lambda_{kl})_{k,l \in [p]}$.
Due to its exponential cost in the number of hyperparameters, grid search can no longer be envisioned.
In this setting, a notable difference with the scalar hyperparameter case is the dimensionality of the terms.
Indeed, the hypergradient $\nabla \cL(\Lb)$ is now represented by a $p \times p$ matrix,
while $\Db$, $\Eb$ and $\hat\Jb$ will be represented by fourth-order tensors in $\bbR^{p \times p \times p \times p}$. 
For simplicity, we compute each element of the matrix $\nabla \cL(\Lb)$ individually as
\begin{align}
  [\nabla \cL(\Lb)]_{kl} = \sum_{i,j=1}^p \frac{\partial \cC}{\partial \Theta_{ij}}(\hat\Thetab(\Lb)) \frac{\partial \hat \Theta_{ij}}{\partial \Lambda_{kl}}(\Lambda_{kl}) \in \bbR\,.
\end{align}
In the matrix case, the function $F$ becomes $F(\Zb, \Lb)
  = \sign(\Zb) \odot (\abs{\Zb} - \gamma \Lb)_+$.
By differentiating the fixed point equation of proximal gradient descent,
\begin{align} \label{eq:fixed_point_matrix}
   & \hat\Thetab(\Lb) =
  F(\underbrace{\hat\Thetab(\Lb) - \gamma (\Sb - \hat\Thetab(\Lb)^{-1})}_{\hat \Zb}, \Lb) \, ,
\end{align}
with respect to $\Lambda_{kl}$, we obtain a Jacobian that can be expressed by a $p \times p$ matrix $[\hat \Jb_{(\Lambda_{kl})}]_{ij} = \frac{\partial \hat \Theta_{ij}}{\partial \Lambda_{kl}}(\Lb)$. It satisfies
\begin{equation*}
  \begin{aligned}
    \hat\Jb_{(\Lambda_{kl})} & =
    \Db:
    \left(\hat\Jb_{(\Lambda_{kl})}
    - \gamma \hat{\Thetab}(\Lb)^{-1}
    \hat\Jb_{(\Lambda_{kl})}
    \hat{\Thetab}(\Lb)^{-1} \right)+ \Eb_{(\Lambda_{kl})}\,. \\
  \end{aligned}
\end{equation*}
Similarly to the scalar case $D_{ijkl} = \mathbbm{1}_{(i,j) = (k,l)} \mathbbm{1}_{|\hat{Z}_{ij}| > \gamma \Lambda_{kl}}$ and $[\Eb_{(\Lambda_{kl})}]_{ij} = -\sign(\hat{Z}_{kl}) \mathbbm{1}_{(i,j) = (k,l)} \mathbbm{1}_{|\hat{Z}_{ij}| > \gamma \Lambda_{kl}}$. The following proposition thus gives the formula for $\hat\Jb_{(\Lambda_{kl})}$.


\begin{proposition} \label{prop:jacob_Lambda}
  Let $\cS \subset[p^2]$ be the set of indices $i$ such that $\vectorize(|\hat \Zb|)_i > \gamma \vectorize(\Lb)_i $.
  The Jacobian $\hat\Jb_{(\Lambda_{kl})}$ is given by
  \begin{equation*}
    \begin{aligned}
       & \vectorize\left(\hat\Jb_{(\Lambda_{kl})}\right)_\cS
      \hspace*{-2mm}= \hspace*{-1mm} \left[ \left( \hat{\Thetab}(\Lb)^{-1} \otimes \hat{\Thetab}(\Lb)^{-1} \right)_{\cS,\cS} \right]^{-1} \hspace*{-4mm} \vectorize\left( \Eb_{(\Lambda_{kl})} \right)_\cS / \gamma, \\
       & \vectorize\left(\hat\Jb_{(\Lambda_{kl})}\right)_{\cS^c} = 0                                                                                                                                                 
      \, .
    \end{aligned}
  \end{equation*}
  The Jacobian of $\hat\Thetab(\Lb)$ with respect to $\Lb$ can be represented by the $\bbR^{p \times p \times p \times p}$ tensor $\hat\Jb$ where $\hat J_{ijkl} = \left([\hat\Jb_{(\Lambda_{kl})}]_{ij}\right)_{i,j,k,l}$
\end{proposition}



We notice that the inverse of the Kronecker product, the bottleneck in the computation of $\hat\Jb$, only has
to be computed once for all $(\Lambda_{kl})_{k,l \in [p]}$. By its expression, $\Eb_{(\Lambda_{kl})}$ is a matrix with a single $\pm1$ element at index $(i,j) = (k,l)$. Therefore $\hat\Jb_{(\Lambda_{kl})}$ is obtained by extracting the only column of $\left[ \left( \hat{\Thetab}(\Lb)^{-1} \otimes \hat{\Thetab}(\Lb)^{-1} \right)_{\cS,\cS} \right]^{-1}$ indexed by that non-zero element.
\vspace{-2mm}
\section{Experiments}\label{sec:experiments}

In this section, we present our proposed methodology for tuning the hyperparameter(s) of the GLASSO, and we aim to address the following three questions through our experiments: 1) How does our approach compare to grid-search? 2) What level of improvement can be achieved by extending to matrix regularization? 3) What are the limitations of our method in its current state?

To answer these questions, we generated synthetic data using the \texttt{make\_sparse\_spd\_matrix} function of \texttt{scikit-learn}, which created a random $100 \times 100$ sparse and positive definite matrix $\Thetab_\mathrm{true}$ by imposing sparsity on its Cholesky factor. We then sample $2000$ points following a Normal distribution $\mathbf{x}_i \sim \mathcal{N}\left(0, \Thetab_\mathrm{true}^{-1}\right), i\in [n]$ \textit{i.i.d.}
\vspace{-3mm}
\paragraph{The criterion and its gradient} Selecting the appropriate criterion $\cC$ to minimize is not an easy task without strong prior knowledge of the true matrix $\Thetab_{\mathrm{true}}$ to be estimated.
In our numerical validation, we use the unpenalized negative likelihood on left-out data.
More precisely, we split the data into a training and testing set with a $50-50$ ratio $(\mathbf{x}_i)_{i \in [n]} = (\mathbf{x}_i)_{i \in I_{\mathrm{train}}} \cup (\mathbf{x}_i)_{i \in I_{\mathrm{test}}}$ and we consider the hold-out criterion
$\cC(\Thetab) = -\logdet(\Thetab) + \langle \Sb_\mathrm{test}, \Thetab\rangle$ where $\Sb_\mathrm{test}= \frac{1}{|I_{\mathrm{test}}|} \sum_{i \in I_{\mathrm{test}}} \mathbf{x}_i \mathbf{x}_i^{\top}$ is the empirical covariance of the test samples (respectively $\Sb_\mathrm{train}$ for the train set).
This corresponds to the negative log-likelihood of the test data under the Gaussian assumption $\forall i \in I_{\mathrm{test}}, \ \mathbf{x}_i \sim \mathcal{N}(0, \Thetab^{-1})$ \textit{i.i.d} \citep{friedman2008sparse}. The intuition behind the use of this criterion is that $\Thetahat$ should solve the GLASSO problem on the training set while remaining plausible on the test set.
Other possible choices include reconstruction errors such as $\cC(\Thetab) = \norm{\Thetab \Sb_\mathrm{test} - \Id}_F$, but a comparison of the effect of the criterion on the solution is beyond the scope of this paper.
In our case, the criterion's gradient $\nabla \cC(\Thetab)$ is then equal to
$-\Thetab^{-1} + \Sb_\mathrm{test}$ \citep[\S A.4.1]{boyd2004convex}.

\vspace{-3mm}
\paragraph{Computing the Jacobian} Based on the previous results we have all the elements at hand to compute the hypergradient for scalar and matrix hyperparameters.
In the first case it reads $\frac{\dr \mathcal{L}}{\dr \lambda}(\lambda) = \langle\hat\Jb, \nabla\cC(\hat\Thetab(\lambda)) \rangle$ with $\hat\Jb$ as in \Cref{prop:jacscalar}, while in the latter case it can be computed with the double contraction $\nabla \mathcal{L}(\Lb) = \hat\Jb : \nabla\cC(\hat\Thetab(\Lb))$ with $\hat\Jb$ as in \Cref{prop:jacob_Lambda}.
In the code, we use the parametrization $\lambda = \exp(\alpha)$ and $\Lambda_{kl} = \exp(\alpha_{kl})$ respectively for the scalar and matrix regularization, and optimize over $\alpha$ in order to impose the positivity constraint on $\lambda$, as in \citep{bertrand2022implicit}.
We rely on the GLASSO solver \citep{laska_narayan_2017_830033} for computing $\hat\Thetab(\cdot)$.
For solving \eqref{eq:bilevel}, we use simple gradient descent with fixed step-size $\rho = 0.1$.





\vspace{-3mm}
\paragraph{Comparison with grid-search}
As a sanity check, we first compare our method with a single hyperparameter (scalar case) to grid search. The initial regularization parameter $\lambda^\mathrm{init}$ is chosen such that the estimated precision matrix $\widehat\Thetab(\lambda^\mathrm{init})$ is a diagonal matrix:
$\lambda^\mathrm{init} = \log(\|\Sb_{\operatorname{train}}\|_{\infty})$.
\Cref{fig:scalar_grad_vs_grid} demonstrates that both methods find the same optimal $\lambda$, which we refer to as $\lambda^\mathrm{opt}_\mathrm{id}$, and that a first-order method that is suitably tuned can swiftly converge to this optimum. We also compute in the same Figure  the relative error (RE) $\frac{\norm{\Thetab_\mathrm{true} - \widehat\Thetab(\lambda)}}{\norm{\Thetab_\mathrm{true}}}$ between the estimation and the true matrix (in blue). We notice that $\hat \Thetab(\lambda^\mathrm{opt}_\mathrm{id})$ results in a slightly worse RE than the optimal one. This highlights the importance of the choice of $\cC$, which may not necessarily reflect the ability to precisely reconstruct the true precision matrix $\Thetab_\mathrm{true}$. Nonetheless, it is important to note that the RE represents an oracle error since, in practical scenarios, we do not have access to $\Thetab_\mathrm{true}$. This raises the essential question of criterion selection, which we defer to future research.

\begin{figure}[t]
  \centering
  \includegraphics[width=\linewidth]{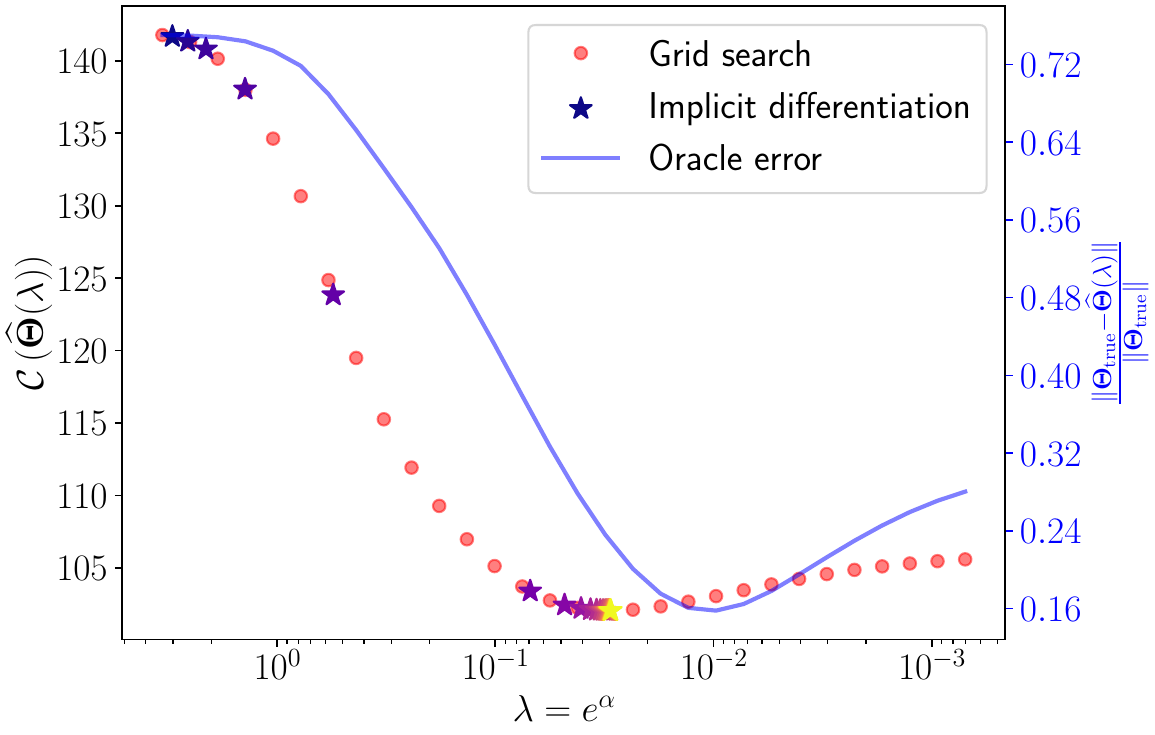}
  \vspace*{-3mm}
  \caption{Value of the criterion $\cC$ \textit{w.r.t.} $\lambda$ for grid-search and our method, along with the oracle RE.}
  \label{fig:scalar_grad_vs_grid}
\end{figure}


\vspace{-3mm}
\paragraph{Matrix regularization}
Our approach demonstrates its value in the context of matrix regularization, where grid search is incapable of identifying the optimal solution within a reasonable amount of time. As depicted in \Cref{fig:matrix_reg}, leveraging matrix regularization with appropriately tuned parameters enhances the value of the bilevel optimization problem.
Furthermore, as demonstrated in \Cref{fig:matrix_visu}, our method successfully modifies each entry $\Lambda_{kl}$ of the regularization matrix, resulting in an estimated matrix $\widehat\Thetab(\Lb^\mathrm{opt})$ that aligns visually with the oracle $\Thetab_\mathrm{true}$.
The edge brought by this improvement remains to be further investigated with respect to the computational cost of the method.
While tuning the step-size, we observed that the non-convexity in this case appears to be more severe.
We speculate that utilizing more sophisticated first-order descent algorithms from the non-convex optimization literature could be more robust than plain gradient descent.

\begin{figure}[t]
  \centering
  \includegraphics[width=\linewidth]{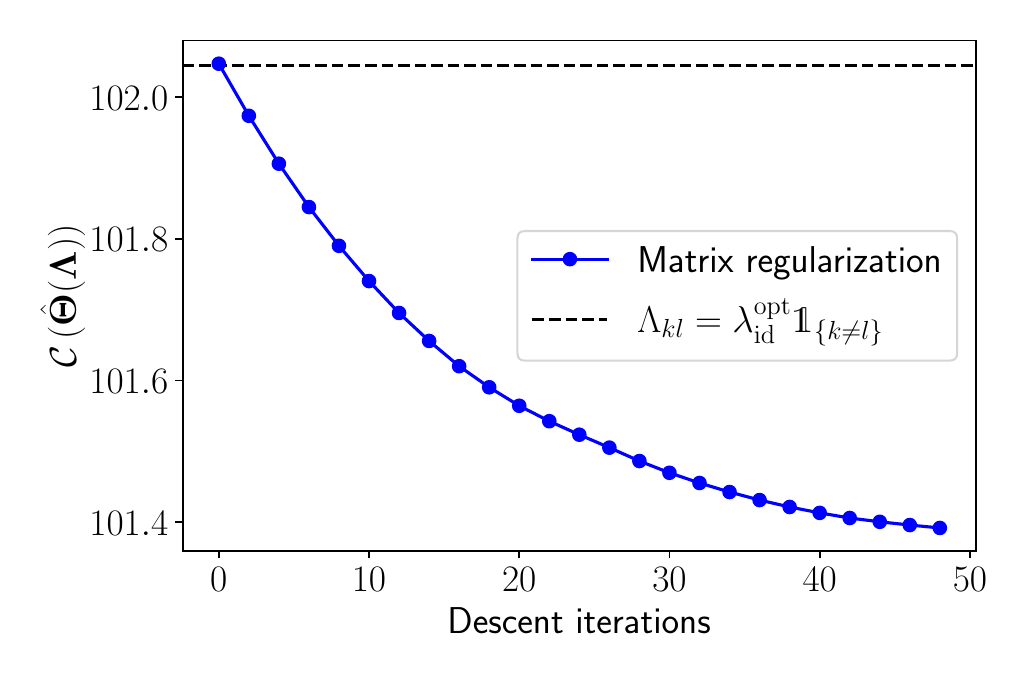}
  \vspace*{-3mm}
  \caption{Outer objective value for the bilevel problem along iterations of hypergradient descent.}
  \label{fig:matrix_reg}
\end{figure}

\begin{figure}[t]
  \centering
  \includegraphics[width=\linewidth]{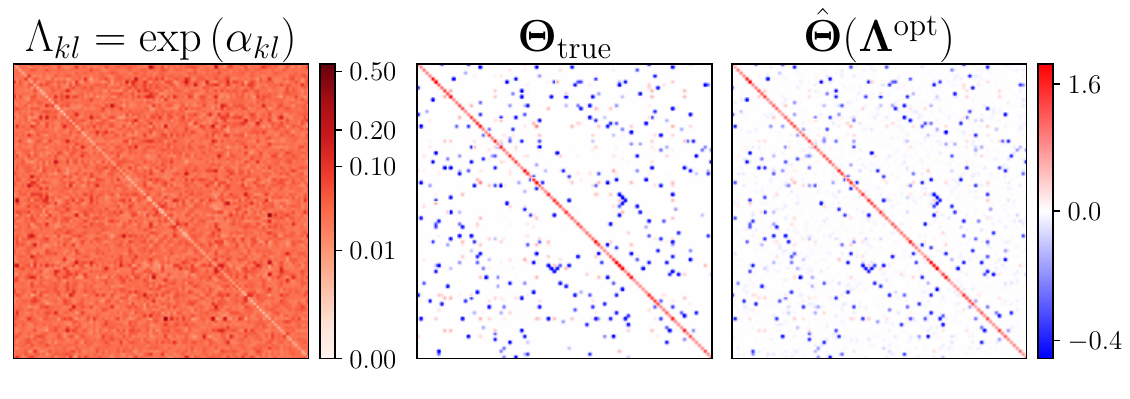}
  \caption{Visualization of the matrices $\Lb^\mathrm{opt}$, $\Thetab_\mathrm{true}$ and $\widehat\Thetab(\Lambda^\mathrm{opt})$.}
  \label{fig:matrix_visu}
  \vspace{-5mm}
\end{figure}


\section*{Conclusion}

In this work, we have proposed a first-order hyperparameter optimization scheme based on implicit differentiation for automatically tuning the GLASSO estimator.
We exploited the sparse structure of the estimated precision matrix for an efficient computation of the Jacobian of the function mapping the hyperparameter to the solution of the GLASSO.
We then proposed an extension of the single regularization parameter case to element-wise (matrix) regularization. As future directions of research, we plan on studying the influence of the criterion $\cC$ on the sparsity of the recovered matrix, as well as clever stepsize tuning strategies for the hypergradient descent.
In the broader sense, we will also benchmark our method against data-based approaches to hyperparameter optimization such as deep unrolling \citep{shrivastava2019glad}.
Finally, we provide high-quality code available freely on GitHub\footnote{\url{https://github.com/Perceptronium/glasso-ho}} for the reproducibility of our experiments.





\setlength{\bibsep}{2pt}
{\footnotesize
  \bibliography{./references.bib}
}

\end{document}